\documentclass[runningheads]{llncs}
\usepackage[T1]{fontenc}
\usepackage[colorlinks=true, linkcolor=blue, citecolor=blue]{hyperref}
\usepackage{graphicx}
\usepackage{color}
\usepackage{amsmath}
\usepackage{amssymb}
\usepackage{multirow}
\usepackage{comment}
\usepackage{marvosym}
\usepackage{float}

\begin{document}
\title{Shadows in the Attention: Contextual Perturbation and Representation Drift in the Dynamics of Hallucination in LLMs}
%
%
\author{Zeyu Wei\inst{1,2}\orcidID{0009-0001-3335-2366} \and
Shuo Wang\inst{1,2}\orcidID{0009-0008-4232-4245} \and
Xiaohui Rong\inst{1}\orcidID{0009-0001-0536-4538} \and
Xuemin Liu\inst{1}\orcidID{0009-0007-8759-1219} \textsuperscript{(\Letter)} \and
He Li\inst{1,2}\orcidID{0009-0004-3862-9936}
}
\authorrunning{Zy. Wei et al.}
%
\institute{Computer Network Information Center, CAS, Beijing, China \and
University of Chinese Academy of Sciences, Beijing, China \\
\email{\{zywei, swang, xhrong, lxm\}@cnic.cn, lihe241@mails.ucas.ac.cn}}
\maketitle              
\begin{abstract}
Hallucinations—plausible yet erroneous outputs—remain a critical barrier to reliable deployment of large language models (LLMs). We present the first systematic study linking hallucination incidence to internal‐state drift induced by incremental context injection. Using TruthfulQA, we construct two 16-round “titration” tracks per question: one appends relevant but partially flawed snippets, the other injects deliberately misleading content. Across six open-source LLMs, we track overt hallucination rates with a tri-perspective detector and covert dynamics via cosine, entropy, JS and Spearman drifts of hidden states and attention maps. Results reveal (1) monotonic growth of hallucination frequency and representation drift that plateaus after 5–7 rounds; (2) relevant context drives deeper semantic assimilation, producing high-confidence “self-consistent” hallucinations, whereas irrelevant context induces topic-drift errors anchored by attention re-routing; and (3) convergence of JS-Drift (\textasciitilde0.69) and Spearman-Drift (\textasciitilde0) marks an “attention-locking” threshold beyond which hallucinations solidify and become resistant to correction. Correlation analyses expose a seesaw between assimilation capacity and attention diffusion, clarifying size-dependent error modes. These findings supply empirical foundations for intrinsic hallucination prediction and context-aware mitigation mechanisms.

\keywords{Hallucinations \and LLMs \and Internal‐State \and Attention-Locking.}
\end{abstract}
\section{Introduction}
LLMs, such as GPT-3 \cite{NEURIPS2020_1457c0d6}, PaLM \cite{chowdhery2023palm}, LLaMA \cite{touvron2023llama}, and GPT-4 \cite{bubeck2023sparksartificialgeneralintelligence}, have approached human-level performance in question-answering tasks. However, they still generate factually incorrect or self-contradictory "hallucinations" \cite{filippova-2020-controlled,maynez-etal-2020-faithfulness}, undermining user trust and creating risks in high-stakes domains like medicine and law; thus hallucinations have become a major obstacle to LLM deployment \cite{farquhar2024detecting,ji2023survey,rawte2023survey}.

Existing mitigation approaches divide into data/knowledge-layer and retrieval-augmentation strategies: TruthfulQA \cite{lin-etal-2022-truthfulqa} evaluates misconception repetition using leading questions, Chen et al. \cite{chen-etal-2024-complex} identify hallucinations through retrieval-verification; Pan et al.'s \cite{pan-etal-2024-automatically} self-critique and self-correction mechanisms and retrieval-augmented generation \cite{gao2024retrievalaugmentedgenerationlargelanguage} also significantly reduce error rates.

Researchers have also looked "inward" for hidden signals. Azaria and Mitchell \cite{azaria-mitchell-2023-internal} discovered that GPT-3's encoder states shift before false outputs; Duan et al. \cite{duan2024llmsknowhallucinationempirical} similarly observed trajectory differences between correct and hallucinated answers, indicating that latent representations partially encode "truthfulness." Wang et al. \cite{wang2025modelsthinkingaboutunderstanding} and Chen et al. \cite{chen2024insidellmsinternalstates} demonstrated that mid-layer attention and feed-forward activations can predict hallucinations, while Huang et al. \cite{10820047} showed high entropy and large sampling variance correlate with increased hallucination risk. Hidden activations and probability distributions have thus become crucial indicators for online detection and error correction.

Despite research progress, systematic understanding of hallucination's internal mechanisms remains lacking. Existing studies often analyze model behavior under fixed prompts or design specific metrics for particular scenarios, without comprehensively comparing how model internal representations drift under context perturbations and how such drift triggers hallucinations.

This research focuses on how context perturbations trigger systematic drift in LLM hidden representations leading to hallucinations. We examine six mainstream models (Llama3-8B, Llama3.2-1B, Falcon3-1B, Falcon3-7B, Qwen2.5-1.5B, Qwen2.5-7B) \cite{llama3,falcon,qwen2.5} performing QA tasks under relevant, seemingly relevant, and misleading contexts, tracking hidden state evolution and its association with hallucinated outputs. Results show that injected contexts form an "attention shadow," gradually shifting model focus, causing representation drift, and ultimately generating incorrect answers. This study doesn't propose new models but identifies and characterizes representation drift patterns related to hallucinations, providing empirical evidence for understanding LLM hallucination internal mechanisms and establishing theoretical foundations for future detection and mitigation strategies.

We contribute three advances: (1) a controlled-context framework that traces hidden-state shifts, revealing LLM internal dynamics; (2) cross-model proof that attention and representation patterns contain universal hallucination precursors; and (3) causal evidence that context-induced representation drift heightens hallucination risk, anchoring future detection and mitigation.

\section{Related Works}
Extensive research aims to detect and mitigate LLM "hallucinations" \cite{huang2025survey}. Existing methods can be categorized into four types: \textbf{(1) Retrieval augmentation and post-correction:} Incorporating external knowledge retrieval or automatic error correction before or after generation to reduce hallucinations \cite{gao2024retrievalaugmentedgenerationlargelanguage}. Pan et al. reviewed self-feedback correction strategies \cite{pan-etal-2024-automatically}; Peng proposed LLM-Augmenter iteratively integrating retrieval and feedback \cite{peng2023check}; Su utilized real-time detection to dynamically adjust retrieved content for entity error correction \cite{su2024mitigating}. \textbf{(2) Uncertainty assessment:} Leveraging output confidence to identify hallucinations, using statistical measures like entropy to detect fabricated content \cite{farquhar2024detecting,fadeeva-etal-2024-fact}; SelfCheckGPT employs generation diversity to self-check inconsistent answers \cite{manakul-etal-2023-selfcheckgpt}; other work fine-tunes models to predict their own accuracy \cite{kadavath2022languagemodelsmostlyknow}. \textbf{(3) Internal representation analysis:} Azaria and Mitchell discovered hidden layers carrying error signals \cite{azaria-mitchell-2023-internal}; Duan compared correct versus hallucinated trajectories and utilized differences to mitigate hallucinations \cite{duan2024llmsknowhallucinationempirical}; Wang decomposed reasoning into understanding-querying-generation phases to extract state features, significantly improving detection without external knowledge \cite{wang2025modelsthinkingaboutunderstanding}. \textbf{(4) Context perturbation and attention manipulation:} Attention "backtracking" alone can identify and mitigate conflict-type hallucinations \cite{chuang-etal-2024-lookback}; Bazarova compared answer-prompt attention topology for unsupervised hallucination detection in retrieval settings \cite{bazarova2025hallucination}; hallucination risk significantly increases when context exceeds model knowledge or sequence order is perturbed \cite{flemings2024characterizing}. However, in-depth research on how LLM internal representation drift leads to hallucinations remains lacking; this paper aims to fill this gap.
\section{Methodology}
This study systematically investigates when and why LLMs hallucinate and how internal-state dynamics precipitate those errors. We pursue two questions: \textbf{(1) How does hallucination probability vary non-linearly with the strength and depth of incremental context injection? (2) Do shifts in hidden representations and attention weights explain these overt hallucination fluctuations?}

Our integrated “overt–covert” pipeline couples a tri-path detector that flags hallucinations with continuous monitoring of hidden vectors and attention matrices. By measuring each stage’s representational drift from a baseline, the framework links observable errors to the underlying dynamics, offering complementary evidence for the generative mechanism of hallucinations.

\subsection{Context-Manipulation Paradigm}
We employ the TruthfulQA-Generation validation set \cite{lin-etal-2022-truthfulqa} and generate $2\times16$ prompt variants per question. Round 0 serves as the baseline, followed by 15 rounds of incremental context injection along two trajectories: \textit{Relevant}, composed primarily of accurate information interspersed with a few superficially credible but subtly incorrect statements; and \textit{Irrelevant}, consisting of misleading or distracting content designed to elicit hallucinations. This titration-style design quantifies how graded contextual interference alters hallucination rates and hidden-state drift.

Guided by this rationale, we implement a dual-track experimental framework. For the relevant track, each question is accompanied by a set of semantically congruent snippets,
\begin{equation}
    \label{add_context1}
    \mathcal{C}_{rel} = \{c_1^{rel}, c_2^{rel}, ..., c_n^{rel}\}
\end{equation}

where $c_i^{\text{rel}}$ denotes the i-th injected relevant fragment. Symmetrically, the irrelevant track furnishes potentially misleading snippets for the same question,
\begin{equation}
    \label{add_context2}
    \mathcal{C}_{irr} = \{c_1^{irr}, c_2^{irr}, ..., c_n^{irr}\}
\end{equation}

with $c_i^{\text{irr}}$ representing the i-th injected irrelevant fragment. Under both conditions, each question undergoes 15 rounds of progressive context augmentation to capture the temporal evolution of model responses; the cumulative context visible at round $t$ is:
\begin{equation}
    \label{add_context3}
    \mathcal{C}_t = \{c_1, c_2, ..., c_t\}
\end{equation}

This finely controlled manipulation paradigm thus allows a systematic assessment of how the depth and nature of external interference jointly influence hallucination incidence and the concomitant drift of internal model states.

\subsection{Hallucination Detection Framework and Text Quality Metrics}
We introduce a tri-perspective consensus framework, augmented with multidimensional quality metrics, for rapid hallucination detection. By cross-validating outputs via three complementary views—\textbf{semantic deviation, factual extension}, and \textbf{logical inference}—the framework reduces the false positives and negatives that plague single-metric methods in complex generation tasks.

\textbf{Semantic deviation detection} measures the semantic distance between a model-generated answer $y$ and a reference answer $r$ in embedding space. The hallucination indicator is defined as:
\begin{equation}
    \label{halu1}
    H_{sem}(y, r) = \mathbb{I}[BERTScore_F(y, r) < \theta_{sem}]
\end{equation}

where $BERTScore_F$ denotes the F1 score based on contextualized embeddings, $\theta_{\text{sem}} = 0.7$ is an empirically chosen threshold, and $\mathbb{I}[\cdot]$ is the indicator function. A low semantic similarity suggests that the model has semantically deviated from the correct answer, signaling a potential hallucination. The theoretical foundation lies in the assumption that greater deviation in semantic space correlates with a higher likelihood of fabricating content. Compared to surface-level string matching, BERTScore leverages deep contextual embeddings to capture subtle semantic discrepancies.

\textbf{Factual extension detection} focuses on whether the model hallucinates concrete factual elements that are absent from the reference set:
\begin{equation}
    \label{halu2}
    H_{ext}(y, \mathcal{R}) = \mathbb{I}[\exists e \in \mathcal{E}(y) : e \notin \mathcal{E}(\mathcal{R})]
\end{equation}

where $\mathcal{E}(\cdot)$ denotes the function that extracts salient factual entities from text, including numbers, named entities, and domain-specific terms. $\mathcal{R}$ is the set of acceptable references. This mechanism effectively captures high-risk hallucinations where the model introduces non-existent facts—instances that may appear credible to users.

\textbf{Logical inference detection} uses a NLI model to assess the logical relationship between the generated answer and the reference set. If the output cannot be logically entailed by any reference, it is likely to contain hallucinatory content:
\begin{equation}
    \label{halu3}
    H_{nli}(y, \mathcal{R}) = \mathbb{I}[\forall r \in \mathcal{R}: NLI(y, r) \in \{\text{neutral}, \text{contradiction}\}]
\end{equation}

where $NLI(y, r)$ is the output of a pretrained NLI model that categorizes the relation as one of: entailment, neutral, or contradiction. A lack of entailment across all references is taken as evidence of a logical-level hallucination. This inference-based strategy complements the semantic and factual components by capturing reasoning failures that might otherwise go undetected.

The overall hallucination label is derived from a logical OR over the three detectors:
\begin{equation}
    \label{halu4}
    H(y, \mathcal{R}) = H_{sem}(y, r_{best}) \vee H_{ext}(y, \mathcal{R}) \vee H_{nli}(y, \mathcal{R})
\end{equation}

where $r_{\text{best}}$ is the most appropriate reference answer.

Based on this unified criterion, we define two hallucination rate metrics. The \textbf{QA-level hallucination rate} measures the proportion of generated answers containing hallucinations:
\begin{equation}
    \label{halu5}
    \text{QA-HallucRate} = \frac{1}{N}\sum_{i=1}^N H(y_i, \mathcal{R}i)
\end{equation}

where $N$ is the total number of question-answer pairs. To further assess hallucination at a finer granularity, we introduce the \textbf{intra-sentence hallucination rate}:
\begin{equation}
    \label{halu6}
    \text{Intra-HallucRate} = \frac{1}{N}\sum_{i=1}^N \frac{1}{|S(y_i)|}\sum_{s \in S(y_i)} H(s, \mathcal{R}i)
\end{equation}

where $S(y_i)$ denotes the set of individual sentences comprising the answer $y_i$. This metric quantifies the average hallucination likelihood at the sentence level and helps determine whether hallucinations are localized to specific sentences or dispersed throughout the output.

To link hallucination to output quality, we measure ROUGE-L, METEOR, and BERTScore. ROUGE-L assesses sequence overlap, METEOR accommodates flexible alignments, and BERTScore-F quantifies deep semantic similarity. Although not used for labeling hallucinations, these metrics reveal whether quality drops co-occur with hallucinations or if fluent text can mask falsehoods, enriching our evaluation of semantic integrity.

\subsection{Internal State Analysis}
\label{sec:INS}
If overt hallucinations in LLMs are indeed triggered by anomalous perturbations in the activation space, then such disruptions should manifest as systematic and measurable changes in internal representations—specifically, in hidden states and attention distributions. To verify this hypothesis, we examine the relationship between hallucination generation and internal dynamics of LLMs from three complementary analytical dimensions: \textbf{representation drift, attention entropy} and \textbf{distributional shift}.

Hidden state vectors encode high-dimensional semantic understanding of the input. Their variations reflect the degree to which contextual information alters the model’s cognitive state. A large representational drift indicates a substantial shift in the model’s internal “thought trajectory,” potentially steering its responses in undesired directions. This drift is quantified using cosine distance:
\begin{equation}
    \label{is1}
    D_{cos}(h_t, h_0) = 1 - \frac{h_t \cdot h_0}{||h_t|| \cdot ||h_0||}
\end{equation}

where $h_0$ denotes the final hidden state (typically of the last token in the final layer) under the zero-context baseline, and $h_t$ denotes the corresponding hidden state after $t$ rounds of contextual injection. The cosine distance $D_{\text{cos}} \in [0,2]$, with larger values indicating greater contextual influence on internal representations.

Attention entropy analysis characterizes the model’s information allocation strategy. An increase in entropy suggests that the model’s attention has become more diffuse—spreading its focus over more tokens, which may indicate uncertainty and a lack of confident reasoning anchors. Conversely, decreased entropy reflects a more concentrated attention distribution, potentially signaling reliance on specific tokens. To quantify this, we compute attention entropy as:
\begin{equation}
    \label{is2}
    H_{attn}(A) = -\sum_i A_i \log A_i
\end{equation}

where $A_i$ denotes attention weights. The entropy shift after contextual injection is given by:
\begin{equation}
    \label{is3}
    D_{ent}(A_t, A_0) = H_{attn}(A_t) - H_{attn}(A_0)
\end{equation}

with $A_0$ and $A_t$ representing the attention distributions before and after $t-round$ context injection, respectively. Positive values indicate attention diffusion; negative values indicate attention focusing.

To more comprehensively capture shifts in attention patterns, we further employ the Jensen–Shannon (JS) divergence, a symmetric and bounded measure of distributional change:
\begin{equation}
    \label{is4}
    D_{JS}(P, Q) = \frac{1}{2}D_{KL}(P||M) + \frac{1}{2}D_{KL}(Q||M),\quad M = \frac{1}{2}(P + Q)
\end{equation}

where $D_{\text{KL}}(P \Vert Q) = \sum_i P(i) \log \frac{P(i)}{Q(i)}$ is the Kullback–Leibler divergence. To enable comparison between distributions of different lengths, we apply zero-padding:
\begin{equation}
    \label{is5}
    P_{pad}(i) = \begin{cases}
P(i) & \text{if } i < |P| \\
\epsilon & \text{otherwise}
\end{cases}
\end{equation}

where $\epsilon$ is a small positive constant (typically $10^{-12}$), followed by renormalization. JS divergence is advantageous due to its symmetry and boundedness within $[0, 1]$, offering a robust estimate of how context reshapes the model’s attention landscape. A higher JS divergence implies a more profound shift in attention allocation, potentially signaling a change in the model’s reasoning path.

In parallel, we calculate the Spearman rank correlation coefficient to assess the stability of attention rankings:
\begin{equation}
    \label{is6}
    \rho = 1 - \frac{6\sum_i d_i^2}{n(n^2-1)}
\end{equation}

where $d_i$ is the difference in rank for the i-th element between two distributions and $n$ is the distribution length. Zero-padding is again applied to maintain consistent dimensionality. The Spearman coefficient $\rho \in [-1, 1]$, where negative values indicate rank inversion and positive values denote alignment. Unlike divergence measures, Spearman correlation focuses exclusively on ordinal relationships, abstracting away from absolute probability magnitudes and highlighting whether the model re-prioritizes its focal elements under different contextual conditions.

This internal state analysis framework establishes a tight coupling between observable model behaviors and quantifiable shifts in internal representations and attention dynamics. It offers a mechanistic lens through which hallucinations can be interpreted and provides empirical foundations for developing targeted mitigation strategies in LLMs.

\section{Experiments}

To systematically investigate the relationship between model internal state evolution and hallucination generation, we designed a two-phase experiment. In the first phase, we selected various architectures including Llama3, Falcon3, and Qwen 2.5 (1B-8B parameters) \cite{llama3,falcon,qwen2.5}, injected context incrementally, and recorded hallucination changes across rounds. External performance was evaluated using QA-HallucRate to assess hallucination frequency and inter-HallucRate to measure whether hallucinations were concentrated or dispersed within responses. The second phase tracked hidden states and attention distributions to analyze the underlying mechanisms of hallucination generation. The two phases complemented each other: first observing behavior, then dissecting mechanisms, comprehensively revealing how internal state evolution influences hallucinations.

\subsection{Dataset and Context Generation}
This research utilizes the TruthfulQA dataset \cite{lin-etal-2022-truthfulqa} as an evaluation benchmark, which was specifically designed to detect hallucination phenomena in LLMs. It comprises 817 questions covering 38 topics, including domains such as health, law, finance, and politics. The distinctive characteristic of TruthfulQA lies in its question design, which tends to guide models toward producing common misconceptions or popular fallacies, thereby effectively assessing models' ability to resist false information and generate truthful responses. To construct the experimental environment, we developed a context generation framework based on the GPT-4o model \cite{bubeck2023sparksartificialgeneralintelligence}, capable of generating two carefully designed context types for each question: (1) relevant context, containing factual information directly related to the question along with seemingly accurate information, without providing complete answers directly; (2) irrelevant context, containing factual information unrelated to the question. All generated contexts underwent secondary human evaluation and modification. This dual-context design enables comprehensive assessment of LLMs' truthfulness performance under different information environments, particularly analyzing the influence mechanisms behind model hallucination generation.

\subsection{Changes in Hallucination Metrics During Context Injection}

Across 15 context-injection rounds, QA-HallucRate rose markedly. Irrelevant snippets drove rates toward 1.0 within a few iterations; relevant snippets, though less potent, also produced a steady climb. Llama3.2-1B began at $\approx0.78$ (relevant) and reached $\approx0.90$ by round 15, whereas Llama3-8B rose from $\approx0.90$ to $\approx0.94$. Falcon and Qwen models mirrored this trend: relevant tracks converged near 90\% and irrelevant tracks saturated early, shrinking the gap between trajectories. Table \ref{halu_m} reports odd-numbered rounds; The complete table is available upon request from the authors.

As hallucination frequency increased, their distribution showed an opposite trend. In relevant scenarios, inter-HallucRate remained stable at approximately 0.60, indicating errors were sporadically distributed; in irrelevant scenarios, it decreased from 0.60+ to about 0.48, showing that models repeatedly adhered to specific false content, resulting in highly concentrated hallucinations. Relevant information provided some factual support, with hallucinations appearing only as local deviations; accumulated irrelevant information caused models to immerse in self-constructed erroneous contexts, producing concentrated misinformation.(See Fig \ref{intraHallu})

Models displayed size-dependent behavior. Larger variants were more context-sensitive, with higher initial and cumulative hallucination rates. For example, Qwen 2.5-7B scored 0.94 in the first irrelevant round versus 0.90 for the 1.5 B; by round 15 (relevant), rates were 0.92 and 0.90, respectively. Falcon3-7B similarly exceeded its 1 B counterpart. Size also shaped dispersion: Llama3-8B’s first-round relevant spread was 0.64, surpassing the 1 B’s 0.60, implying more scattered minor errors; after repeated irrelevant injections, the 8 B dropped to 0.615 (errors concentrating), whereas the 1 B rose to 0.61 (still diffuse). Thus, larger models initially integrate context broadly—yielding diverse deviations—then converge on specific false narratives under sustained irrelevant input. In sum, incremental context sharply elevates hallucination risk: irrelevant snippets rapidly enforce near-certain, clustered errors, while relevant snippets lessen severity yet still foster frequent, dispersed minor faults over time.

\begin{table}[H]
\caption{Hallucination Metrics Across Different Models and Rounds}
\label{halu_m}
\scriptsize
\resizebox{\textwidth}{!}{%
\begin{tabular}{|l|c|c|c|c|c|c|c|c|c|c|c|}
\hline
\multirow{2}{*}{Model} & \multirow{2}{*}{Round} & \multicolumn{2}{c|}{ROUGE-L} & \multicolumn{2}{c|}{METEOR} & \multicolumn{2}{c|}{BERT-F1} & \multicolumn{2}{c|}{QA-HallucRate} & \multicolumn{2}{c|}{Intra-HallucRate} \\
\cline{3-12}
& & Rel & Irr & Rel & Irr & Rel & Irr & Rel & Irr & Rel & Irr \\
\hline
\multirow{7}{*}{Qwen2.5-7B} & 1 & 0.1037 & 0.0672 & 0.2657 & 0.1779 & 0.8500 & 0.8296 & 0.8600 & 0.9400 & 0.5945 & 0.6270 \\
& 3 & 0.0585 & 0.0362 & 0.1652 & 0.1162 & 0.8325 & 0.8194 & 0.9000 & 0.9800 & 0.6141 & 0.5436 \\
& 5 & 0.0396 & 0.0250 & 0.1177 & 0.0844 & 0.8225 & 0.8119 & 0.9400 & 1.0000 & 0.6092 & 0.5239 \\
& 7 & 0.0303 & 0.0197 & 0.0923 & 0.0665 & 0.8228 & 0.8119 & 0.9200 & 1.0000 & 0.6172 & 0.5040 \\
& 9 & 0.0241 & 0.0153 & 0.0761 & 0.0542 & 0.8228 & 0.8119 & 0.9200 & 1.0000 & 0.6073 & 0.4986 \\
& 11 & 0.0204 & 0.0133 & 0.0642 & 0.0479 & 0.8228 & 0.8119 & 0.9000 & 1.0000 & 0.6096 & 0.4944 \\
& 15 & 0.0156 & 0.0100 & 0.0497 & 0.0354 & 0.8228 & 0.8119 & 0.9200 & 1.0000 & 0.6115 & 0.4838 \\
\hline
\multirow{7}{*}{Qwen2.5-1.5B} & 1 & 0.1084 & 0.0648 & 0.2645 & 0.1787 & 0.8525 & 0.8265 & 0.8600 & 0.9000 & 0.6039 & 0.6106 \\
& 3 & 0.0567 & 0.0350 & 0.1621 & 0.1056 & 0.8331 & 0.8172 & 0.9000 & 1.0000 & 0.5998 & 0.5352 \\
& 5 & 0.0412 & 0.0242 & 0.1238 & 0.0808 & 0.8234 & 0.8119 & 0.9000 & 1.0000 & 0.6154 & 0.5142 \\
& 7 & 0.0318 & 0.0186 & 0.0996 & 0.0666 & 0.8228 & 0.8119 & 0.8800 & 1.0000 & 0.6174 & 0.4945 \\
& 9 & 0.0253 & 0.0145 & 0.0816 & 0.0517 & 0.8228 & 0.8119 & 0.9000 & 1.0000 & 0.6095 & 0.4980 \\
& 11 & 0.0211 & 0.0127 & 0.0712 & 0.0467 & 0.8228 & 0.8119 & 0.9000 & 1.0000 & 0.6164 & 0.4935 \\
& 15 & 0.0158 & 0.0092 & 0.0539 & 0.0339 & 0.8228 & 0.8119 & 0.9000 & 1.0000 & 0.6076 & 0.4791 \\
\hline
\multirow{7}{*}{Falcon3-7B} & 1 & 0.1197 & 0.0766 & 0.3051 & 0.2053 & 0.8566 & 0.8349 & 0.8200 & 0.9200 & 0.5889 & 0.5981 \\
& 3 & 0.0615 & 0.0378 & 0.1843 & 0.1225 & 0.8352 & 0.8212 & 0.8800 & 0.9800 & 0.6163 & 0.5341 \\
& 5 & 0.0418 & 0.0258 & 0.1308 & 0.0855 & 0.8230 & 0.8119 & 0.9000 & 1.0000 & 0.6113 & 0.5124 \\
& 7 & 0.0320 & 0.0186 & 0.1073 & 0.0664 & 0.8228 & 0.8119 & 0.8800 & 1.0000 & 0.6130 & 0.4932 \\
& 9 & 0.0256 & 0.0158 & 0.0837 & 0.0582 & 0.8228 & 0.8119 & 0.9200 & 1.0000 & 0.6220 & 0.5026 \\
& 11 & 0.0221 & 0.0144 & 0.0710 & 0.0502 & 0.8228 & 0.8119 & 0.9800 & 1.0000 & 0.6225 & 0.4940 \\
& 15 & 0.0176 & 0.0104 & 0.0538 & 0.0331 & 0.8228 & 0.8119 & 0.9000 & 1.0000 & 0.6113 & 0.4779 \\
\hline
\multirow{7}{*}{Falcon3-1B} & 1 & 0.1210 & 0.0725 & 0.2952 & 0.1960 & 0.8554 & 0.8322 & 0.8000 & 0.9000 & 0.5946 & 0.5779 \\
& 3 & 0.0597 & 0.0368 & 0.1667 & 0.1156 & 0.8335 & 0.8188 & 0.9000 & 0.9800 & 0.6120 & 0.5273 \\
& 5 & 0.0408 & 0.0236 & 0.1191 & 0.0761 & 0.8229 & 0.8119 & 0.8600 & 1.0000 & 0.6074 & 0.4993 \\
& 7 & 0.0311 & 0.0176 & 0.0978 & 0.0604 & 0.8228 & 0.8119 & 0.9000 & 1.0000 & 0.6129 & 0.4849 \\
& 9 & 0.0256 & 0.0144 & 0.0792 & 0.0508 & 0.8228 & 0.8119 & 0.9000 & 1.0000 & 0.6118 & 0.4865 \\
& 11 & 0.0212 & 0.0119 & 0.0666 & 0.0413 & 0.8228 & 0.8119 & 0.9000 & 1.0000 & 0.6116 & 0.4815 \\
& 15 & 0.0162 & 0.0087 & 0.0517 & 0.0312 & 0.8228 & 0.8119 & 0.9000 & 1.0000 & 0.6112 & 0.4783 \\
\hline
\multirow{7}{*}{Llama3-8B} & 1 & 0.0869 & 0.0685 & 0.2176 & 0.1826 & 0.8387 & 0.8277 & 0.9000 & 0.9600 & 0.6398 & 0.6034 \\
& 3 & 0.0533 & 0.0362 & 0.1515 & 0.1101 & 0.8234 & 0.8164 & 0.8600 & 1.0000 & 0.6368 & 0.5432 \\
& 5 & 0.0379 & 0.0248 & 0.1176 & 0.0846 & 0.8212 & 0.8119 & 0.8800 & 1.0000 & 0.6206 & 0.5278 \\
& 7 & 0.0318 & 0.0186 & 0.0996 & 0.0666 & 0.8228 & 0.8119 & 0.8800 & 1.0000 & 0.6174 & 0.4945 \\
& 9 & 0.0256 & 0.0151 & 0.0816 & 0.0553 & 0.8228 & 0.8119 & 0.9000 & 1.0000 & 0.6095 & 0.4980 \\
& 11 & 0.0214 & 0.0127 & 0.0712 & 0.0467 & 0.8228 & 0.8119 & 0.9000 & 1.0000 & 0.6164 & 0.4935 \\
& 15 & 0.0158 & 0.0092 & 0.0539 & 0.0339 & 0.8228 & 0.8119 & 0.9000 & 1.0000 & 0.6076 & 0.4791 \\
\hline
\multirow{7}{*}{Llama3.2-1B} & 1 & 0.1465 & 0.0689 & 0.3035 & 0.1774 & 0.8619 & 0.8251 & 0.7800 & 0.8800 & 0.6018 & 0.5943 \\
& 3 & 0.0673 & 0.0364 & 0.1791 & 0.1109 & 0.8398 & 0.8166 & 0.9000 & 1.0000 & 0.6022 & 0.5195 \\
& 5 & 0.0451 & 0.0237 & 0.1283 & 0.0792 & 0.8244 & 0.8119 & 0.8800 & 1.0000 & 0.6081 & 0.5031 \\
& 7 & 0.0336 & 0.0176 & 0.0982 & 0.0610 & 0.8228 & 0.8119 & 0.9000 & 1.0000 & 0.6110 & 0.4878 \\
& 9 & 0.0265 & 0.0139 & 0.0807 & 0.0490 & 0.8228 & 0.8119 & 0.9000 & 1.0000 & 0.6104 & 0.4855 \\
& 11 & 0.0220 & 0.0117 & 0.0680 & 0.0411 & 0.8228 & 0.8119 & 0.9000 & 1.0000 & 0.6140 & 0.4808 \\
& 15 & 0.0165 & 0.0087 & 0.0524 & 0.0309 & 0.8228 & 0.8119 & 0.9000 & 1.0000 & 0.6101 & 0.4784 \\
\hline
\end{tabular}
}
\vspace{-2em}
\end{table}

\begin{figure}
\includegraphics[width=\textwidth]{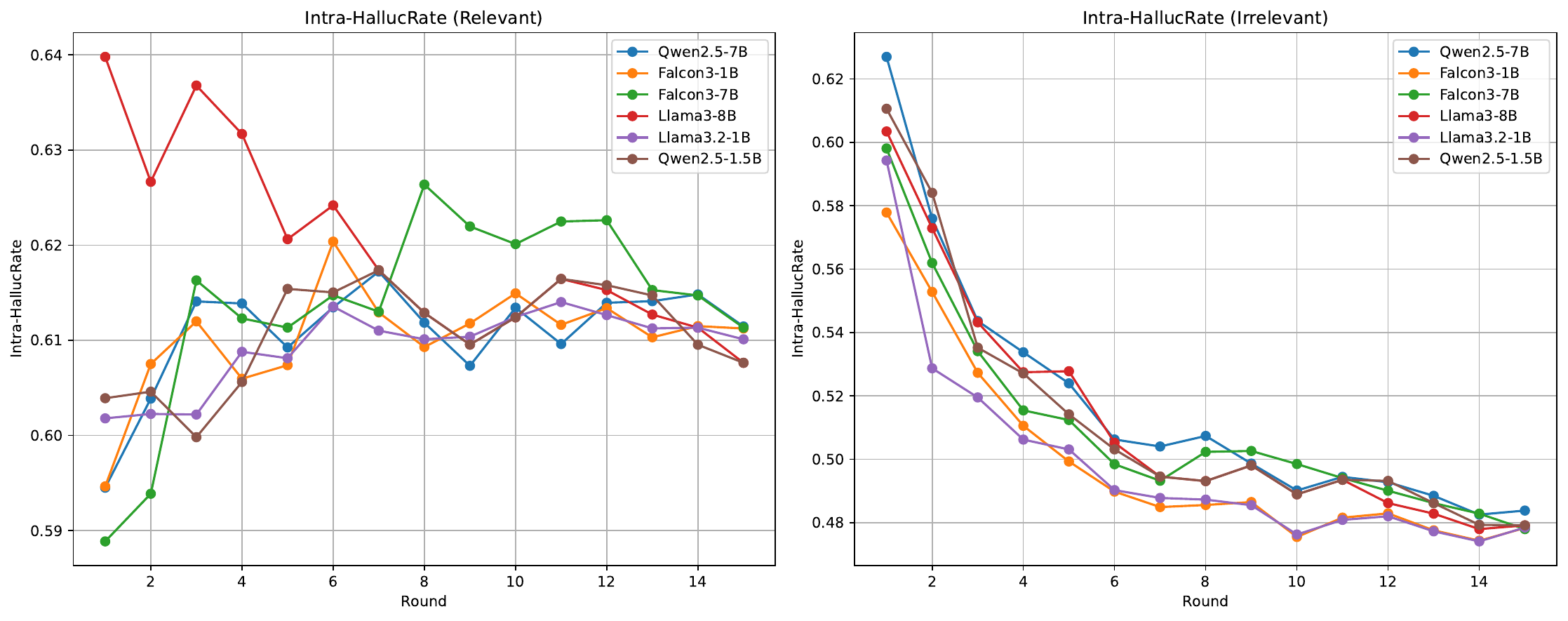}
\vspace{-2em}
\caption{Intra-Hallucination Rate Comparison Across Different Models} \label{intraHallu}
\vspace{-2em}
\end{figure}

\subsection{Analysis of Model Internal Drift and Hallucination Dynamics Mechanisms}

To analyze the deeper causes of hallucination phenomena, we examined model internal states, comparing drift degrees in hidden layer representations and attention distributions under different context conditions, exploring hallucination evolution processes and their association with model scale. Table \ref{att_m} presents the data used in this section. For brevity, only odd-numbered rounds are included(The complete table is available upon request from the authors).

We found that across all models, four internal drift metrics (Cos-Drift, Ent-Drift, JS-Drift, Spearman-Drift, We have already introduced them in the \ref{sec:INS} section) exhibited increasing trends with successive rounds before stabilizing after approximately 5-7 rounds. For example, in the Llama-3 8B model under relevant context, Cos-Drift increased from 0.1925 in round 1 to 0.2116 in round 15, while Ent-Drift grew from 0.6196 to 1.4643 (See Fig \ref{llama8b}).

\begin{figure}
\vspace{-2em}
\includegraphics[width=\textwidth]{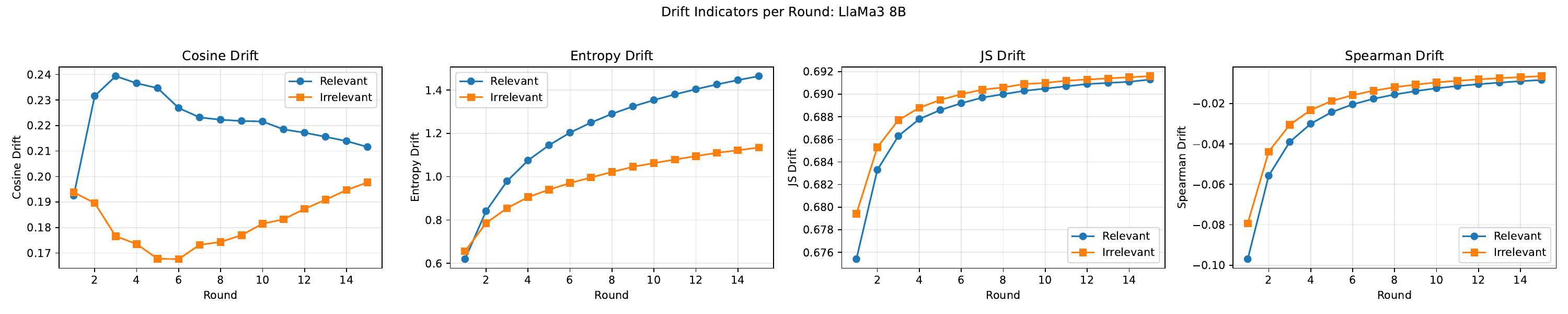}
\vspace{-2em}
\caption{Internal State Drift Metrics of the Llama3-8B Model} \label{llama8b}
\vspace{-4em}
\end{figure}

\begin{figure}
\includegraphics[width=\textwidth]{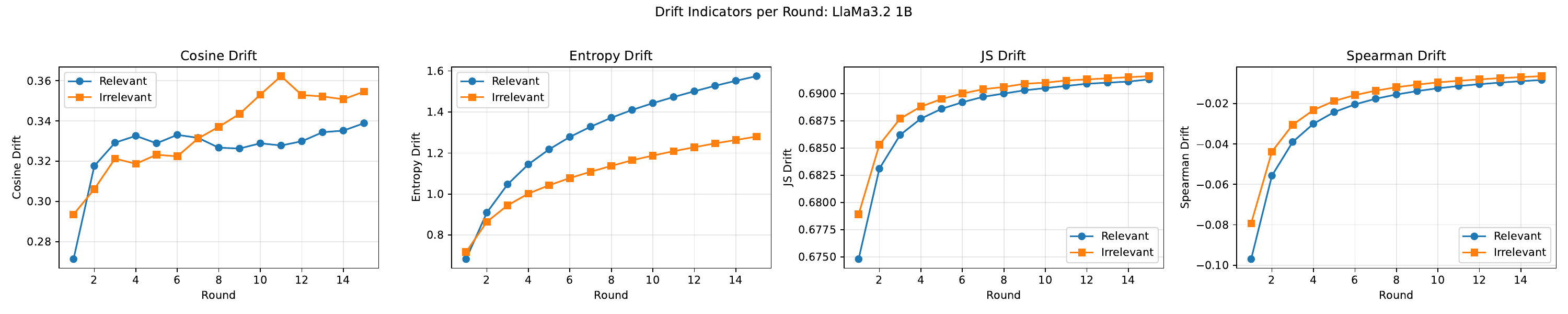}
\vspace{-2em}
\caption{Internal State Drift Metrics of the Llama3.2-1B Model} \label{llama1b}
\vspace{-1.8em}
\end{figure}

\begin{corollary}
    Internal state drift is a universal phenomenon in large language models under accumulated long-context conditions, and hallucination risk gradually increases with cumulative internal drift.
\end{corollary}

This indicates that long-sequence contexts induce systematic shifts in model hidden states and attention distributions. As dialogue rounds increase, the model's assimilation of context deepens, with its internal "beliefs" continuously deviating from initial semantic anchors, providing a dynamical foundation for hallucination generation.

Comparing relevant and irrelevant contexts reveals: under relevant contexts, Cos-Drift and Ent-Drift average values are generally higher than under irrelevant contexts. For example, in round 15 of the Llama-3 8B model, Cos-Drift reached 0.2116 under relevant context versus 0.1977 under irrelevant context; Ent-Drift was 1.4643 under relevant context compared to 1.1347 under irrelevant context. Under irrelevant context conditions, JS-Drift and Spearman-Drift stabilized more rapidly with smaller variance fluctuations.

\begin{table}[H]
\caption{Internal States Drift Metrics Across Different Models and Rounds}
\label{att_m}
\scriptsize
\resizebox{\textwidth}{!}{%
\begin{tabular}{|l|c|c|c|c|c|c|c|c|c|}
\hline
\multirow{2}{*}{Model} & \multirow{2}{*}{Round} & \multicolumn{2}{c|}{Cos-Drift} & \multicolumn{2}{c|}{Ent-Drift} & \multicolumn{2}{c|}{JS-Drift} & \multicolumn{2}{c|}{Spearman-Drift} \\
\cline{3-10}
& & Rel & Irr & Rel & Irr & Rel & Irr & Rel & Irr \\
\hline
\multirow{7}{*}{Qwen2.5-7B} & 1 & 0.1584 & 0.2864 & 0.8749 & 0.9425 & 0.6748 & 0.6786 & -0.0934 & -0.0764 \\
& 3 & 0.2245 & 0.3022 & 1.3501 & 1.2301 & 0.6861 & 0.6876 & -0.0368 & -0.0288 \\
& 5 & 0.2460 & 0.3027 & 1.5692 & 1.3522 & 0.6886 & 0.6895 & -0.0228 & -0.0177 \\
& 7 & 0.2374 & 0.2698 & 1.7127 & 1.4337 & 0.6897 & 0.6903 & -0.0167 & -0.0128 \\
& 9 & 0.2342 & 0.3067 & 1.8197 & 1.5059 & 0.6903 & 0.6908 & -0.0130 & -0.0100 \\
& 11 & 0.2420 & 0.2633 & 1.9031 & 1.5580 & 0.6907 & 0.6912 & -0.0107 & -0.0082 \\
& 15 & 0.2277 & 0.2726 & 2.0324 & 1.6449 & 0.6913 & 0.6916 & -0.0078 & -0.0061 \\
\hline
\multirow{7}{*}{Qwen2.5-1.5B} & 1 & 0.1547 & 0.1390 & 0.8469 & 0.9020 & 0.6745 & 0.6785 & -0.0934 & -0.0764 \\
& 3 & 0.1317 & 0.1212 & 1.3075 & 1.1810 & 0.6861 & 0.6876 & -0.0368 & -0.0288 \\
& 5 & 0.1185 & 0.0863 & 1.5200 & 1.3042 & 0.6886 & 0.6895 & -0.0228 & -0.0177 \\
& 7 & 0.1036 & 0.0709 & 1.6575 & 1.3851 & 0.6896 & 0.6903 & -0.0167 & -0.0128 \\
& 9 & 0.0995 & 0.0631 & 1.7594 & 1.4560 & 0.6903 & 0.6908 & -0.0130 & -0.0100 \\
& 11 & 0.1042 & 0.0649 & 1.8379 & 1.5054 & 0.6907 & 0.6912 & -0.0107 & -0.0082 \\
& 15 & 0.1151 & 0.0666 & 1.9589 & 1.5874 & 0.6913 & 0.6916 & -0.0078 & -0.0061 \\
\hline
\multirow{7}{*}{Falcon3-7B} & 1 & 0.3431 & 0.3614 & 1.1180 & 1.1887 & 0.6741 & 0.6774 & -0.0928 & -0.0763 \\
& 3 & 0.3227 & 0.3214 & 1.6267 & 1.5176 & 0.6863 & 0.6875 & -0.0371 & -0.0294 \\
& 5 & 0.3203 & 0.3147 & 1.8526 & 1.6563 & 0.6885 & 0.6894 & -0.0230 & -0.0181 \\
& 7 & 0.3055 & 0.3080 & 1.9966 & 1.7477 & 0.6897 & 0.6903 & -0.0168 & -0.0131 \\
& 9 & 0.3028 & 0.3253 & 2.1039 & 1.8252 & 0.6903 & 0.6908 & -0.0131 & -0.0102 \\
& 11 & 0.3016 & 0.3251 & 2.1870 & 1.8809 & 0.6907 & 0.6912 & -0.0107 & -0.0084 \\
& 15 & 0.3072 & 0.3431 & 2.3168 & 1.9728 & 0.6912 & 0.6916 & -0.0079 & -0.0062 \\
\hline
\multirow{7}{*}{Falcon3-1B} & 1 & 0.3589 & 0.3629 & 0.9915 & 1.0776 & 0.6737 & 0.6772 & -0.0927 & -0.0763 \\
& 3 & 0.3551 & 0.3891 & 1.5198 & 1.4330 & 0.6862 & 0.6875 & -0.0371 & -0.0294 \\
& 5 & 0.4008 & 0.3425 & 1.7650 & 1.5913 & 0.6885 & 0.6894 & -0.0230 & -0.0181 \\
& 7 & 0.4004 & 0.3706 & 1.9260 & 1.6976 & 0.6897 & 0.6903 & -0.0168 & -0.0131 \\
& 9 & 0.4624 & 0.3289 & 2.0461 & 1.7863 & 0.6903 & 0.6908 & -0.0131 & -0.0102 \\
& 11 & 0.4197 & 0.2880 & 2.1382 & 1.8523 & 0.6907 & 0.6912 & -0.0107 & -0.0084 \\
& 15 & 0.3718 & 0.2410 & 2.2809 & 1.9616 & 0.6912 & 0.6916 & -0.0079 & -0.0062 \\
\hline
\multirow{7}{*}{Llama3-8B} & 1 & 0.1925 & 0.1938 & 0.6196 & 0.6564 & 0.6754 & 0.6794 & -0.0970 & -0.0793 \\
& 3 & 0.2394 & 0.1766 & 0.9793 & 0.8544 & 0.6863 & 0.6877 & -0.0390 & -0.0305 \\
& 5 & 0.2347 & 0.1677 & 1.1455 & 0.9401 & 0.6886 & 0.6895 & -0.0242 & -0.0187 \\
& 7 & 0.2232 & 0.1732 & 1.2499 & 0.9963 & 0.6897 & 0.6904 & -0.0176 & -0.0136 \\
& 9 & 0.2218 & 0.1770 & 1.3244 & 1.0456 & 0.6903 & 0.6909 & -0.0138 & -0.0106 \\
& 11 & 0.2185 & 0.1832 & 1.3800 & 1.0791 & 0.6907 & 0.6912 & -0.0113 & -0.0087 \\
& 15 & 0.2116 & 0.1977 & 1.4643 & 1.1347 & 0.6913 & 0.6916 & -0.0083 & -0.0064 \\
\hline
\multirow{7}{*}{Llama3.2-1B} & 1 & 0.2713 & 0.2935 & 0.6825 & 0.7169 & 0.6748 & 0.6789 & -0.0970 & -0.0793 \\
& 3 & 0.3292 & 0.3213 & 1.0470 & 0.9446 & 0.6862 & 0.6877 & -0.0390 & -0.0305 \\
& 5 & 0.3289 & 0.3232 & 1.2176 & 1.0426 & 0.6886 & 0.6895 & -0.0242 & -0.0187 \\
& 7 & 0.3316 & 0.3313 & 1.3281 & 1.1084 & 0.6897 & 0.6904 & -0.0176 & -0.0136 \\
& 9 & 0.3263 & 0.3435 & 1.4099 & 1.1646 & 0.6903 & 0.6909 & -0.0138 & -0.0106 \\
& 11 & 0.3278 & 0.3623 & 1.4729 & 1.2087 & 0.6907 & 0.6912 & -0.0113 & -0.0087 \\
& 15 & 0.3389 & 0.3546 & 1.5748 & 1.2801 & 0.6913 & 0.6916 & -0.0083 & -0.0064 \\
\hline
\end{tabular}
}
\vspace{-2.5em}
\end{table}

\begin{corollary}
    Different types of contexts drive different forms of hallucination mechanisms: relevant contexts promote deep semantic assimilation producing high-confidence self-consistent hallucinations, while irrelevant contexts induce concatenation-type or topic-drift hallucinations through information routing perturbations and anchor reordering.
\end{corollary}

Relevant contexts cause models to substantially adjust their internal representations within semantic space; if models absorb information containing errors, the resulting hallucinations demonstrate extreme coherence and self-consistency. In contrast, models under irrelevant stimuli tend to reorganize attention routing and semantic anchor sequences, forming "concatenation-style" outputs manifesting as tangential responses or topic deviations, but with lower internal consistency.

Observations indicate that under relevant context conditions, Cos-Drift and Ent-Drift variances are generally higher than under irrelevant contexts, while JS-Drift and Spearman-Drift variances remain relatively similar across both conditions.

\begin{corollary}
    High volatility (high variance) in internal drift metrics often predicts the occurrence of high-confidence hallucinations, while steady-state low variance suggests the model's hallucinated content has become solidified.
\end{corollary}

Metric variance reflects model internal state stability. Relevant contexts trigger larger-scale, higher-fluctuation adjustments, potentially producing high-confidence, self-consistent hallucinations resistant to correction. Irrelevant contexts cause models to quickly enter a low-variance "immune" steady state, where erroneous information becomes new attention anchors resistant to local corrections.

Across all models, JS-Drift typically reaches stable saturation values (approximately 0.690±0.001) within 6-8 rounds, while Spearman-Drift simultaneously converges near zero, indicating rapid reshaping of internal attention structures (See Fig \ref{lock}). This "synchronous convergence" phenomenon shows significant negative correlation with hallucination correctability ($\rho \approx -0.66$).

\begin{figure}
\includegraphics[width=\textwidth]{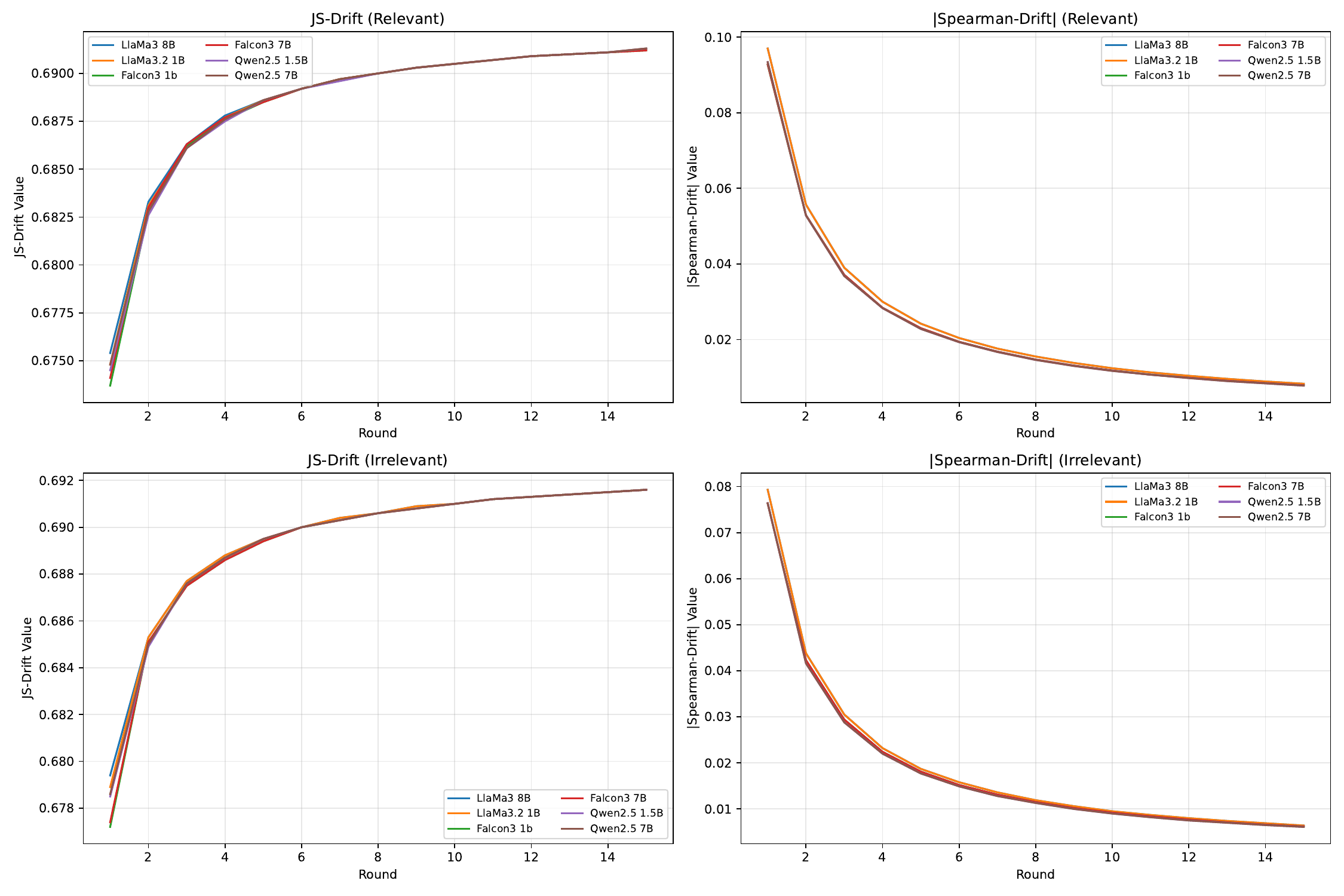}
\vspace{-2em}
\caption{Comparison of JS-Drift and Spearman-Drift under Relevant vs. Irrelevant Context Injection} \label{lock}
\vspace{-1.8em}
\end{figure}

\begin{corollary}
    Once a model's internal attention structure enters a "locking threshold" (JS-Drift saturation and Spearman-Drift approaching zero), generated hallucinations enter a solidified state, making conventional dialogue correction methods ineffective.
\end{corollary}

When attention pattern reshaping reaches steady state, the model enters a new stable information routing phase. Subsequently, the effectiveness of injected corrective information significantly decreases, making it difficult to change previous erroneous "beliefs." This explains the "hallucination solidification" phenomenon: once incorrect routing paths and factual anchors form, the model continues producing factually incorrect but structurally coherent content, requiring radical measures like truncating dialogue context or inserting explicit counterfactual instructions to break incorrect outputs.

Comparing model scale influence reveals larger models show more pronounced differences in Cos-Drift means between relevant/irrelevant contexts. Llama-3 8B maintains $\Delta \mathrm{Cos}$ of approximately 0.05–0.07, Falcon-3 7B shows about 0.02, while smaller models (like Llama-3.2 1B) display differences approaching zero or even negative values.

\begin{corollary}
    Model capacity determines semantic selectivity capability: larger models better distinguish useful information from noise interference; insufficient-capacity models tend to indiscriminately absorb all context, more readily producing hallucinations with random biases.
\end{corollary}

Large models can significantly assimilate highly relevant information fragments while suppressing noise context absorption; however, if relevant fragments contain biases, resulting hallucinations also demonstrate extremely high confidence. Smaller models with limited representation space treat relevant and irrelevant information equally, causing intermingled semantic anchors and producing low-confidence, structurally loose "patchwork" hallucinations.

Analyzing the relationship between Ent-Drift growth slopes and contemporaneous $\Delta \mathrm{Cos}$ across six models revealed significant negative correlation ($\rho \approx -0.71,\ p < 0.01$). In Qwen-2.5 1.5B and Llama-3.2 1B, when $\Delta \mathrm{Cos}$ approached 0, Ent-Drift slopes exceeded 0.11/round.

\begin{corollary}
    A "seesaw" compensatory mechanism exists between semantic assimilation and attention divergence; when models cannot deeply assimilate context, they sacrifice attention focus to expand retrieval range, thus increasing the probability of concatenation-type hallucinations.
\end{corollary}

Small models with insufficient assimilation capacity compensate for comprehension gaps by increasing Ent-Drift, but attention dispersion weakens constraints on key factual anchors, facilitating topic drift. Large models, having sufficiently absorbed context information through Cos-Drift, maintain focused attention, avoiding topic drift, but may assimilate incorrect information more thoroughly, potentially producing highly self-consistent hallucinated content.

\section{Conclusion}
This work provides the first large-scale, cross-model evidence that hallucinations in LLMs arise from systematic, context-induced representation drift. By titrating relevant and irrelevant snippets into TruthfulQA prompts, we show that hallucination rates rise monotonically and saturate once attention topology “locks,” while internal cosine, entropy, JS and Spearman drifts converge. The resulting overt–covert correlations hold across six model families and scales, revealing size-dependent trade-offs between semantic assimilation and attention diffusion. These findings establish robust internal precursors to hallucination, furnishing empirical foundations for intrinsic detection, context-aware mitigation and future architecture-level safeguards.

\bibliographystyle{splncs04}
\bibliography{mybibliography}

\end{document}